\documentclass[12pt]{amsart}
\usepackage[utf8]{inputenc}
\usepackage[english]{babel}
\usepackage{macros}
\usepackage{graphicx} 
\usepackage{siunitx}
\usepackage{standalone}
\usepackage{listings}

\usepackage{mathrsfs}
\usepackage{float} 
\usepackage{mdframed}
\lstdefinelanguage{Julia}%
  {morekeywords={abstract,break,case,catch,const,continue,do,else,elseif,%
      end,export,false,for,function,immutable,import,importall,if,in,%
      macro,module,otherwise,quote,return,switch,true,try,type,typealias,%
      using,while},%
   sensitive=true,%
   alsoother={\$},%
   morecomment=[l]\#,%
   morecomment=[n]{\#=}{=\#},%
   morestring=[s]{"}{"},%
   morestring=[m]{'}{'},%
}[keywords,comments,strings]

\usepackage{setspace}
\usepackage{hyperref}

\hypersetup{
    colorlinks=true,
    linkcolor={orange!90!black},
    citecolor={blue},
    urlcolor={blue}
}

\usepackage{fullpage}
\usepackage{newfloat}
\DeclareFloatingEnvironment[
    fileext=lol,
    listname={List},
    name=List,
    placement=tbhp,
]{mylist}

\newenvironment{List}[1][]{%
  \begin{mylist}[#1]
  \begin{enumerate}[label=($\mathcal{R}$\arabic*)]
}{%
  \end{enumerate}
  \end{mylist}
}

\usepackage[backend=biber,style=trad-plain,backref,backrefstyle=none, doi = false, isbn = false, eprint = false, url=false]{biblatex}
\AtEveryBibitem{\clearlist{language}}
\usepackage[linesnumbered, algoruled, boxed, lined]{algorithm2e}
\SetKwInOut{Input}{Input}
\SetKwInOut{Output}{Output}
\SetKwIF{If}{ElseIf}{Else}{if}{then}{else if}{else}{}
\SetKwFor{For}{for}{do}{}

\usepackage{csquotes}

\usepackage[nameinlink]{cleveref}

\crefname{mylist}{list}{lists}
\Crefname{mylist}{List}{Lists}

\numberwithin{equation}{section}
\theoremstyle{plain}
\newtheorem{theorem}{Theorem}[section]

\newtheorem{corollary}[theorem]{Corollary}
\newtheorem{lemma}[theorem]{Lemma}
\newtheorem{proposition}[theorem]{Proposition}

\theoremstyle{definition}
\newtheorem{remark}[theorem]{Remark}
\newtheorem{examplex}[theorem]{Example}
\newenvironment{example}
  {\pushQED{\qed}\examplex}
  {\popQED\endexamplex}

\newtheorem{definition}[theorem]{Definition}

\newenvironment{boxy}
  {\begin{mdframed}[skipabove=10pt,skipbelow=10pt]}
  {\end{mdframed}}

\newcommand{\cg}{\mathcal{G}}

\newcommand{\ch}{\mathcal{H}}

\DeclareMathOperator{\pa}{pa}

\DeclareMathOperator{\an}{an}
\DeclareMathOperator{\An}{An}

\DeclareMathOperator{\de}{de}
\DeclareMathOperator{\De}{De}
\DeclareMathOperator{\neighbors}{ne}

\DeclareMathOperator{\indeg}{indeg}

\DeclareMathOperator{\glob}{global}
\newcommand{\indep}{\perp\!\!\!\perp}

\DeclareMathOperator{\skel}{skel}

\newcommand{\CIperp}{\mathrel{\text{$\perp\mkern-10mu\perp$}}}

\RequirePackage{ifthen}
\RequirePackage{listofitems}
\newcommand{\CI}[1]{{%
  \setsepchar{{:}/{|}/{,}}
  \ignoreemptyitems
  \readlist*\mylist{#1}
  \ifthenelse{\listlen\mylist[] = 2}{\mylist[2]}{}%
  [\mylist[1,1,1] \CIperp \ifthenelse{\listlen\mylist[1,1] = 2}{\mylist[1,1,2]}{\mylist[1,1,1]}%
  \ifthenelse{\listlen\mylist[1] = 2}{{} \mid \mylist[1,2]}{}]%
}}

\newcommand{\dsep}{\perp_\mathrm{d}}
\newcommand{\dsepgiven}[3]{[#1 \dsep #2 | #3]}
\newcommand{\globaldsep}[1]{\glob_\mathrm{d}(#1)}
\newcommand{\equivd}{\sim_\mathrm{d}}
\newcommand{\equivast}{\sim_*}
\newcommand{\starsep}{\perp_*}
\newcommand{\starsepgiven}[3]{[#1 \starsep #2 | #3]}
\newcommand{\globalstarsep}[1]{\glob_\ast(#1)}
\newcommand{\equivstar}{\sim_\ast}

\newcommand{\Csep}{\perp_{C^\ast}}

\newcommand{\Csepgiven}[3]{[#1 \Csep #2 | #3]}

\newcommand{\globalCsep}[2]{\glob_{C^\ast}(#1, #2)}

\newcommand{\Csepvar}[1]{{%
  \setsepchar{{:}/{|}/{,}}
  \ignoreemptyitems
  \readlist*\mylist{#1}
  \ifthenelse{\listlen\mylist[] = 2}{\mylist[2]}{}%
  [\mylist[1,1,1] \Csep \ifthenelse{\listlen\mylist[1,1] = 2}{\mylist[1,1,2]}{\mylist[1,1,1]}%
  \ifthenelse{\listlen\mylist[1] = 2}{{} \mid \mylist[1,2]}{}]%
}}

\newcommand{\cgtr}{\mathcal{G}^{tr}_C}
\newcommand{\critdag}[1]{\mathcal{G}^{\ast}_{#1}}

\newcommand{\picrit}[3][\pi]{{#1}^{#2 #3}_{\text{crit}}}

\begin{document}

\title[A PC Algorithm for Max-Linear Bayesian Networks]{A PC Algorithm for Max-Linear Bayesian Networks}

\author{Carlos Améndola}
\address{Technische Universität Berlin}
\email{amendola@math.tu-berlin.de}

\author{Benjamin Hollering}
\address{Max Planck Institute for Mathematics in the Sciences}
\email{benjamin.hollering@mis.mpg.de}

\author{Francesco Nowell}
\address{Technische Universität Berlin}
\email{nowell@math.tu-berlin.de}

\date{}

\onehalfspace
\begin{abstract}
Max-linear Bayesian networks (MLBNs) are a relatively recent class of structural equation models which arise when the random variables involved have heavy-tailed distributions. Unlike most directed graphical models, MLBNs are typically not faithful to d-separation and thus classical causal discovery algorithms such as the PC algorithm or greedy equivalence search can not be used to accurately recover the true graph structure. In this paper, we begin the study of constraint-based discovery algorithms for MLBNs given an oracle for testing conditional independence in the true, unknown graph. We show that if the oracle is given by the $\ast$-separation criteria in the true graph, then the PC algorithm remains consistent despite the presence of additional CI statements implied by $\ast$-separation. We also introduce a new causal discovery algorithm named \texttt{PCstar} which assumes faithfulness to $C^\ast$-separation and is able to orient additional edges which cannot be oriented with only d- or $\ast$-separation. 

\end{abstract}
\maketitle

\section{Introduction}

Structural equation models on directed acyclic graphs are often used to model the causal relationships between the components of a random vector $X = (X_1, 
\ldots, X_n)$. In a structural equation model, the component $X_i$ of $X$ is typically a function of its parents plus independent noise of the form
\(
X_i = f_i(X_{\pa(i)}, \epsilon_i), 
\)
where $\pa(i)$ is the set of parents of $i$. 
Structural equation models, sometimes called structural causal models or Bayesian networks, have been studied extensively over the past 30 years, especially in the context of the emerging field of \emph{causality} \cite{Pearl_2009} and have been used in a variety of applications including biology  and economics \cite{Durbin_Eddy_Krogh_Mitchison_1998, EVANS1999135, handbook-of-graphical-models}. 
These models are particularly well understood in the case that the independent errors $\epsilon_i$ are Gaussian and the $f_i$ are linear \cite{Lauritzen_Graphical_Models, handbook-of-graphical-models}, though significant progress has recently been made when the former assumption is dropped as well \cite{tramontano2024parameteridentificationlinearnongaussian, wang2020high}.

In this paper, we focus on a relatively new variation of structural causal models, known as \emph{max-linear models} or max-linear Bayesian networks (MLBN) in which the random vector $X$ follows the structural equation system
\begin{align} \label{eqn:mlbn}
X_i =\bigvee_{j \in \pa(i)}c_{ji}X_j \vee Z_i, ~ \quad ~ c_{ij}, Z_i \geq 0,
\end{align}

where $\vee = \max$, the $c_{ij}$ are edge weights, and the $Z_i$ are independent, atom-free, continuous random variables. These models were first introduced in \cite{gissibl_max-linear_2018} to model scenarios with large risk, meaning that the variables $Z_i$ are typically assumed to be heavy-tailed and thus often take extreme values. Max-linear models have since become one of the core approaches for modeling causal relationships in extreme scenarios alongside the parametric setting described in \cite{engelke2025extremesstructuralcausalmodels}. Max-linear models excel at modeling scenarios which involve cascading failure in which one extreme event causes other extreme events throughout the network as often seen in applications such as financial risk \cite{einmahl2018continuous} or ecology \cite{asadi2015extremes}. 

One fundamental problem concerning any structural causal model is that of \emph{causal discovery} in which the goal is to reconstruct a directed graph whose edges encode causal relationships between random variables using data. This problem has been studied extensively for most classical families of directed graphical models and has led to the development of many efficient causal discovery algorithms \cite{discovery_review}. These causal discovery algorithms generally fit into two paradigms known as \emph{score-based} algorithms and \emph{constraint-based} algorithms. Score-based algorithms such as Greedy Equivalence Search (GES) \cite{chickering2003optimal} attempt to optimize a score-function while moving between graphs using a restricted set of edge additions, deletions, and reversals. On the other hand, constraint-based algorithms such as the PC algorithm \cite{spirtes_algorithm_1991} attempt to remove and orient edges such that the resulting graph satisfies exactly the conditional independence statements which are present in the observed data. In almost all existing causal discovery algorithms, \emph{faithfulness} to d-separation is required for consistency. Unlike most directed graphical models, MLBNs are not generically faithful to the well known d-separation criteria \cite{geiger_d-separation_2013}, meaning that a full dimensional subset of distributions in the model satisfy additional conditional independence (CI) statements which are not implied by d-separation on the underlying graph. However, \cite{amendola_conditional_2022} introduced the stronger criteria of $\ast$-separation and $C^\ast$-separation, which capture the additional CI statements which MLBNs satisfy and showed that MLBNs are generically faithful to the latter. 

In this paper, we investigate constraint-based causal discovery algorithms for MLBNs. This means we assume that there exists a true underlying DAG $\cg$ with weight matrix $C$ and assume we have an oracle for testing conditional independence in the observed data. We first show that if the input to the PC algorithm is the global Markov property of $\ast$-separation on $\cg$, that is the oracle produces the set of CI statements $\globalstarsep{\cg} = \bigl\{ \CI{i,j|K} \ \ \  \text{s.t} \ \  \starsepgiven{i}{j}{K} \ \text{holds in} \ \cg \bigr\}$, then the output will be identical to the result of running the PC algorithm with input $\glob_d(\cg)$. That is, despite the fact that $\ast$-separation yields additional CI statements, and thus MLBNs are not faithful to d-separation, the PC algorithm remains consistent and still produces the same result. We then turn our attention to constraint-based causal discovery with an oracle for checking $C^\ast$-separation in the pair $(\cg, C)$ since this is the criteria to which MLBNs are actually faithful, and thus if the oracle were to be replaced with an actual conditional independence test, this would be the true input. We introduce a new algorithm named \texttt{PCstar} which takes as input an oracle for testing $C^\ast$-separation in the true unknown pair $(\cg, C)$ and returns a CPDAG, $\mathcal{H}$ ,which faithfully encodes all conditional independencies yielded by the oracle. In particular, there exists an orientation $\mathcal{H'}$ of $\mathcal{H}$ and a positive dimensional set of coefficient matrices on $\mathcal{H}'$ such that $(\mathcal{H}', C_{\mathcal{H}'})$ and $(\cg, C_\cg)$ yield the same set of statements under $C^\ast$-separation. Our algorithm is directly inspired by the PC algorithm and proceeds in three steps, which are \emph{skeleton retrieval}, \emph{collider detection}, and \emph{cycle orientation}. Our first two steps are identical to that of the PC algorithm, albeit with the assumption of faithfulness to $C^\ast$-separation instead of d-separation. We then show how the additional statements which come from $C^\ast$-separation can be used to orient additional edges which lie in certain cycles. Lastly, we show that similarly to the PC algorithm, there exists a variant of our algorithm which is polynomial if the maximum in-degree of any node in the true graph is bounded. We end by showcasing how our algorithm performs on sparse random graphs. 

The remainder of this paper is organized as follows. In \Cref{sec:prelim} we provide background on max-linear models, separation criteria, and the PC algorithm. In \Cref{unweightedPC}, we show that the PC algorithm returns the same output when given an oracle for $\ast$-separation as it does for d-separation oracle despite the fact that the $\ast$-separation oracle applied to the unknown true graph yields additional CI statements. In \Cref{sec:pcstar}, we introduce our new causal discovery algorithm for max-linear models, named \texttt{PCstar} and show that the output faithfully encodes all CI statements produced by the oracle.  We then show that when the maximum in-degree of the true DAG $\cg$, is bounded, \texttt{PCstar} is polynomial in the number of nodes $n$ and that the extra statements which come from $C^\ast$-separation can often be used to orient additional edges which cannot be oriented under the assumption of faithfulness to d-separation. In \Cref{sec:simulations}, we show how \texttt{PCstar} performs on sparse random graphs and discuss some directions for future work.

\section{Preliminaries}
\label{sec:prelim}

\subsection{Graph terminology}
A \emph{simple directed graph} $\cg=(V,E)$ consists of nodes $V$ and directed edges $E$, i.e., edges of the form $i\rightarrow j$ such that there is at most one edge between any two nodes. We say that two nodes $i$ and $j$ are adjacent if they are connected by an edge. The set of all nodes adjacent to $i$ are the \emph{neighbors} of $i$, denoted $\neighbors(i)$.
An $i-j$ \emph{path} $\pi$ is a sequence of distinct nodes $(i, \dots, j)$ such that all consecutive nodes in $\pi$ are adjacent. We say $\pi = (k_0 =i, \dots , k_s = j)$ is \emph{directed} if all edges on $\pi$ are of the form $k
_i \rightarrow k_{i+1}$, i.e, point towards $j$. An $i-j$ path $\pi$ is called a \emph{trek} if there exists a $k\in \pi$ such that the subpaths $(i, \dots k)$ and $(k, \dots, j)$ are directed paths from $k$ to $i$ and $k$ to $j$ respectively. 
A directed path from $i$ to $j$ and the edge $j \rightarrow i$ form a \emph{directed cycle}. We call a simple directed graph without directed cycles a \emph{directed acyclic graph} (DAG). Given an edge $i\rightarrow j$ we say that $i$ is a \emph{parent} of $j$ and denote the set of parents of $j$ with $\pa(j)$. The cardinality of $\pa(j)$ is the \emph{in-degree} of $j$, denoted $\indeg(j)$. Given a directed path $i \rightarrow \dots \rightarrow j$ we say that $i$ is an \emph{ancestor} of $j$, respectively $j$ is a \emph{descendant} of $i$. We denote the set of ancestors with $\an(j)$ and the set of descendants with $\de(i)$. For a subset of nodes $K \subset V$ we define $\an(K) := \cup_{k\in K} \an(k) \setminus K$ and $\An(K) := \an(K) \cup K $. The sets $\de(K)$ and $\De(K)$ are defined analogously.  \\
The \emph{undirected skeleton} of a directed graph $\cg = (V,E)$ is the undirected graph with node set $V$ containing the edge $i-j$ if and only if $i$ and $j$ are adjacent in $\cg$, and is denoted $\skel(\cg)$. The notions of path, neighbors and cycles in undirected graphs are defined analogously to the directed case. 

\subsection{Graphs and separation conditions} 
Graphical models identify conditional independence relations through a separation condition $\perp$. The best-known instance of this in directed acyclic graphical models is \emph{d-separation} for linear structural equation models. A triple of nodes $\{i,k,j\} \subset V$ form a \emph{ collider} in a DAG $\cg = (V,E)$ if $i \rightarrow k \in E$ and $j \rightarrow k \in E$. The collider is said to be \emph{unshielded} if $i \rightarrow j \notin E$ and $j \rightarrow i \notin E$. An $i-j$ path $\pi$ in a DAG is \emph{d-connecting} given $K \subset V\setminus{ij}$ if all colliders along $\pi$ lie in $\An(K)$ and no non-collider along $\pi$ lies in $K$. 
If a d-connecting path exists, we say that $i$ and $j$ are \emph{d-connected} given $K$. Otherwise, we say that they are \emph{d-separated} and write $\dsepgiven{i}{j}{K}$. \\
 It is well-known \cite[Theorem 2]{geiger_d-separation_2013} that linear structural equation models on a DAG $\cg$ are \textit{faithful} to d-separation in the sense that for all $i,j, K \subset V$:
\begin{align}
    \CI{X_i, X_j | X_K} \ \ \text{holds in $X$} \ \ \iff \ \  \dsepgiven{i}{j}{K} \ \ \text{holds in $\cg$}.
\end{align}
Equivalently, the complete set of conditional independence statements in the model can be recovered from its \textit{Global Markov property}, defined as
\begin{align} \label{def:globalmarkov}
        \glob_{d}(\cg) = \bigl\{ \CI{i,j|K} \ \ \  \text{s.t} \ \  \dsepgiven{i}{j}{K} \ \text{holds in} \ \cg \bigr\}. 
\end{align}
The faithfulness property is what allows the reconstruction of all of the probabilistic CI statements in a distribution from the purely combinatorial properties of $\cg$ and vice versa. 

Two graphs $\cg$ and $\mathcal{H}$ are said to be \textit{Markov equivalent} with respect to a d-separation if $\glob_{d}(\cg) = \glob_d(\mathcal{H})$. We write this as $\cg \sim_d \mathcal{H}$. This relation partitions the set of DAGs on $n$ nodes in the following way: 
\begin{theorem}\cite[Theorem 2.1]{andersson_characterization_1997} \label{thm:dsepmarkoveq}
    Two DAGs $\cg$ and $\ch$ with equal node set are d-Markov equivalent if and only if they have the same undirected skeleton and unshielded colliders.
\end{theorem}
\subsection{Causal discovery and the PC algorithm} \phantom{ } \\
The term \textit{causal discovery} is used to refer to the following task: 
\begin{boxy}
    \centering 
    \textit{Given data which is known to come from a graphical model $(\cg, X)$, recover $\cg$.}
\end{boxy} 
The PC algorithm of Spirtes and Glymour \cite{spirtes_algorithm_1991} is a constraint-based algorithm which outputs a graph approximating the true DAG $\cg$ from conditional independence statements. We describe the two steps of the algorithm under the idealized assumption that there exists a perfect test for asserting whether a given conditional independence statement $\CI{X_i, X_j |X_K}$ lies in $\globaldsep{\cg}$ (or equivalently, that the entire set $\globaldsep{\cg}$ is known). The first step is the \emph{skeleton learning step}, which retrieves the undirected skeleton of the true DAG $G := \skel(\cg)$ by querying the set $\globaldsep{\cg}$. 
\begin{algorithm}
\caption{Skeleton Learning Step of the PC Algorithm}
\label{alg:PCalgskel}
\Input{A complete set of CI statements $\globaldsep{\cg}$ coming from a graphical model faithful to d-separation}  \Output{The undirected skeleton of $\cg$}

Set $G$ to be the complete graph on $n$ nodes

\For{$\ell$ from $0$ to $\max_{(i,j) \in H} |\neighbors(\{i, j\})|$}{
    \For{each edge $(i, j) \in G$}{
        \For{each $K \subseteq \neighbors(i) \cup \neighbors(j) \setminus \{i, j\}$ s.t. $|K| = \ell$}{
            \If{$(i,k) \in G$ for all $k \in K$ \textbf{or} $(j, k) \in G$ for all $k \in K$}{
            \If{$[i \indep j | K] \in \globaldsep{\cg}$}{
                remove edge $(i,j)$ from $G$ \\
                Set $Sepset(i,j) := K$
            }
            }
        }
    }
}
\end{algorithm}
A proof of correctness of \Cref{alg:PCalgskel} may be found in \cite{geiger_d-separation_2013}. An important feature is present on line 2; in order to retrieve the adjacencies of the true DAG $\cg$ it is sufficient to query $\globaldsep{\cg}$ for CI statements $[i \indep j |K]$ in which the size of $K$ is bounded by the maximal in-degree of $\cg$. This is because any two non-adjacent nodes $i,j \in V$ are d- separated by a subset of either $\pa(i)$ or $\pa(j)$ \cite[Lemma 1]{verma_equivalence_2022}.
In particular, this ensures that the worst case performance of \Cref{alg:PCalgskel} is $n^{d+2}$, where $d$ is the maximal in-degree of a node in $\cg$, see \cite[p. 138]{meek_graphical_2023}. \\
After completing the skeleton learning step, the PC algorithm iterates through the edges of the undirected skeleton of $\cg$ and orients them according to the four rules stated in \Cref{misc:orientationrules}. We refer to this as the \emph{edge orientation step} of the PC algorithm. More details may be found in \cite{meek_graphical_2023}. It makes use of the sets $Sepset(i,j)$ collected in \Cref{alg:PCalgskel}, which for each  non-adjacent pair $\{i,j\}$ record a set $K$ of minimal cardinality such that $\CI{i,j|K} \in \globaldsep{\cg}.$
Specifically, the PC algorithm first iterates through the \emph{unshielded triples} of $G$ (i.e. triples of nodes $\{i,j,k\}$ such that $i-j$ and $j-k$ are edges of $G$, but $i-k$ is not) and orients them according to $\mathcal{R}1$. It then proceeds to apply the remaining rules, which are the only choices of edge directions which do not create directed cycles and/or additional unshielded colliders in the graph. The complexity of this procedure lies in $\mathcal{O}(n^3)$ \cite[p. 143]{meek_graphical_2023}, and as such remains bounded above by $n^{d+2}$. 
\begin{List}[H]
    \item For any unshielded triple $i - k - j $, orient the edges as $i \rightarrow k \leftarrow j$ \\ if $k \notin Sepset(i,j)$ holds.
    \item For any remaining unshielded triple oriented $i \rightarrow k - j$, orient it as $i \rightarrow k \rightarrow j$.
    \item For any shielded triple where $i \rightarrow k \rightarrow j$, orient $i-j$ as $i \rightarrow j$.
    \item For any four-node \enquote{diamond structure} such as the one depicted on the right in \Cref{fig:orientationrules}, orient $i-\ell$ as $i \rightarrow \ell$ if $j - \ell - k$ is oriented $j \rightarrow \ell \leftarrow k$. 
    \caption{The 4 edge orientation rules of PC}
     \label{misc:orientationrules}
\end{List}

\begin{figure}
    \begin{subfigure}[t]{0.4\linewidth}
        \centering 
        \begin{tikzpicture}[
    ->,
    >=stealth',
    shorten >= 1pt,
    node distance = 2cm and 1.6cm,
    thick,
    every node/.style={circle, draw, line width=1pt, minimum size=0.65cm, inner sep=0pt},
    every path/.style={thick, -{Stealth[length=8pt, width=8pt, scale=1]}}
]

\node (i) at (0, 2) {$i$};
\node (j) at (2, 2) {$j$};
\node (k) at (1, 0) {$k$};

\path (i) edge (k);
\path (k) edge (j);
\draw[-] (i) -- (j);

\node[draw=none, inner sep=0pt] at (2.8, 1) {\Large$\Rightarrow$};

\node (i2) at (3.4, 2) {$i$};
\node (j2) at (5.4, 2) {$j$};
\node (k2) at (4.4, 0) {$k$};

\path (i2) edge (j2);
\path (i2) edge (k2);
\path (k2) edge (j2);

\end{tikzpicture}
    \end{subfigure}
    \hspace{1cm}
    \begin{subfigure}[t]{0.4\linewidth}
        \centering 
        \begin{tikzpicture}[
    ->,
    >=stealth',
    shorten >= 1pt,
    node distance = 2cm and 1.6cm,
    thick,
    every node/.style={circle, draw, line width=1pt, minimum size=0.65cm, inner sep=0pt},
    every path/.style={thick, -{Stealth[length=8pt, width=8pt, scale=1]}}
]

\node (i) at (2, 2) {$i$};
\node (j) at (1, 1) {$j$};
\node (k) at (3, 1) {$k$};
\node (l) at (2, 0) {$\ell$};

\draw[-] (i) -- (j);
\draw[-] (i) -- (k);
\path (j) edge (l);
\path (k) edge (l);
\draw[-] (i) -- (l);

\node[draw=none, inner sep=0pt] at (4, 1) {\Large$\Rightarrow$};

\node (i2) at (6, 2) {$i$};
\node (j2) at (5, 1) {$j$};
\node (k2) at (7, 1) {$k$};
\node (l2) at (6, 0) {$\ell$};

\draw[-] (i2) -- (j2);
\draw[-] (i2) -- (k2);
\path (j2) edge (l2);
\path (k2) edge (l2);
\path (i2) edge (l2);

\end{tikzpicture}
    \end{subfigure}
    \caption{Depiction of the orientation rules $\mathcal{R}3$ and $\mathcal{R}4$.}
    \label{fig:orientationrules}
\end{figure}
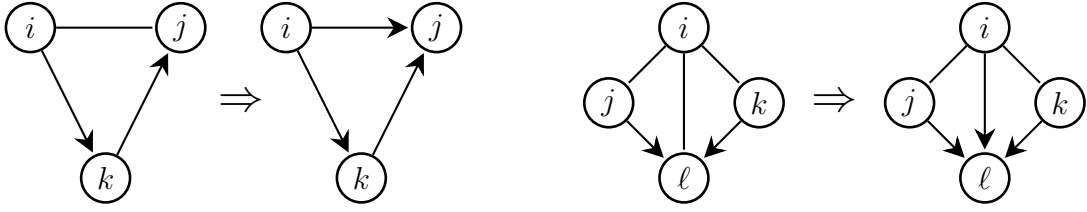
The output of the edge orientation step is a \textit{complete partially oriented directed acyclic graph}, or CPDAG. It has the same unshielded colliders as the true DAG $\cg$ and thus represents the Markov equivalence class of $\cg$ due to \Cref{thm:dsepmarkoveq}. It is the best approximation of $\cg$ that can be obtained from the conditional independence information of a linear structural equation model on $\cg$  without additional interventional data \cite[Theorem 3]{meek_causal_2013}.

\subsection{Star separation criteria for MLBNs}

Max-Linear Bayesian networks are generally not faithful to d-separation; more precisely,  additional conditional independence may hold in the model which is not implied by d-separation. 

This motivated the introduction of stronger separation conditions for Max-Linear Bayesian Networks. This was done in \cite{amendola_conditional_2022}. We recall the relevant notions and definitions of both weighted and unweighted \emph{star separation}. 
\subsection{Unweighted $\ast$-separation}

We adopt the characterization of $\starsep$ which was developed in \cite{amendola21markov}. An undirected $i-j$ path in a DAG $\cg$ is \emph{$\ast$-connecting} given $K \subset V \setminus {ij}$ if it is d-connecting and contains at most one collider.
If no such path in $\cg$ exists, we say that $i$ and $j$ are \emph{$\ast$-separated} given $K$, and denote this by $\starsepgiven{i}{j}{K}$.
\begin{example}\cite[Example 4.2]{amendola_conditional_2022} \label{eg:cassio}
    \begin{figure}[h!]
    \centering
    \begin{tikzpicture}[
    ->,
    >=stealth',
    shorten >= 1pt,
    thick,
    every node/.style={circle, draw, line width=1pt, minimum size=0.65cm, inner sep=0pt},
    every path/.style={thick, -{Stealth[length=8pt, width=8pt, scale=1]}}
]

\node (1) at (0, 2) {$1$};
\node (2) at (2, 2) {$2$};
\node (3) at (4, 2) {$3$};

\node[fill=orange!40!white] (4) at (1, 0) {$4$};
\node[fill=orange!40!white] (5) at (3, 0) {$5$};

\path (1) edge (4);
\path (2) edge (4);
\path (2) edge (5);
\path (3) edge (5);

\end{tikzpicture}
    \caption{The Cassiopeia graph}
    \label{fig:cassio}
\end{figure}
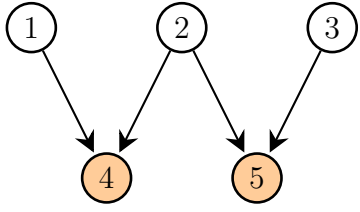
    In \Cref{fig:cassio}, the nodes $1$ and $2$ are $\ast$-connected given $K = \{4,5\}$. $1$ and $3$ are $\ast-$separated given the same $K$. In the max-linear model on this DAG for any choice of coefficients $c_{ji}$, the CI statement $\CI{1,3|45}$ holds. 
\end{example}
From the definition of $\starsep$ it is cleat that $\dsepgiven{i}{j}{K}$ implies $\starsepgiven{i}{j}{K}$ for all $i,j,K$, and that the converse implication also holds when $|K| \leq 1$. \\
On the level of the global Markov properties as in \eqref{def:globalmarkov}, this means that 
    \begin{align} \label{eq:dstarglobalmarkov}
        \globaldsep{\cg} \subset \globalstarsep{\cg}. 
    \end{align}

\Cref{eg:cassio} is an instance where strict inclusion in \eqref{eq:dstarglobalmarkov} holds. Despite this phenomenon the corresponding relations $\sim_d$ and $\sim_\ast$ partition the set of all DAGs into the same Markov equivalence classes.

\begin{theorem}[\cite{amendola21markov}, Theorem 3.4] \label{thm:dstarmarkovclass}
For any two DAGs $\cg$ and $\ch$: $\ch \equivd \cg \iff \ch \equivstar \cg$.

\end{theorem}
The behavior of the PC algorithm under $\starsep$ is investigated in \Cref{unweightedPC}.
\subsection{The weighted case: C*-separation}
We now consider the more general setting of \emph{weighted} Max-Linear Bayesian Networks. Let $\cg$ be a DAG on $n$ nodes with edge set $E$ and denote the set of coefficient matrices supported on $\cg$ by~$\mathbb{R}^E$; this is the set of all $n \times n$ matrices $C$ with $c_{ij} = 0$ if $i \to j \notin E$ and $c_{ij} \in \mathbb{R}_{> 0}$ otherwise. We call a pair $(\cg,C)$ a \emph{weighted DAG}. A random vector $X$ is distributed according to the MLBN on $(\cg, C)$ if it fulfills \eqref{eqn:mlbn}. The introduction of edge weights brings with it additional challenges, which we illustrate with two examples. Firstly, the set of CI statements which hold in a weighted MLBN depend on the choice of weight matrix $C$. 
\begin{example} \label{eg:3nodeDAG}
Consider the complete weighted DAG $(\cg,C)$ on 3 nodes.

\begin{figure}[H]
\centering
\begin{tikzpicture}
    \begin{scope}[->, -Stealth, every node/.style={circle,draw},line width=1pt, node distance=1.8cm]
    \node (1) {$1$};
    \node (2) [right of=1] {$2$};
    \foreach \from/\to in {1/2}
    \draw (\from) -- (\to);
    \path[every node/.style={font=\sffamily\small}]
    (1) -- (2) node [midway, above] {$c_{12}$};
    \node (3) [ right of=2] {$3$};
    \foreach \from/\to in {2/3}
    \draw (\from) -- (\to);
    \path[every node/.style={font=\small}]
    (2) -- (3) node [midway, above] {$c_{23}$};
    \foreach \from/\to in {1/3}
    \draw (\from) to[bend right=40] (\to);
    \path[every node/.style={font=\small}]
    (1) -- (3) node [midway, yshift = -30pt] {$c_{13}$};

\end{scope}
\end{tikzpicture}

\caption{A complete weighted DAG on 3 nodes.}
\label{fig:3nodecompleteDAG}
\end{figure}
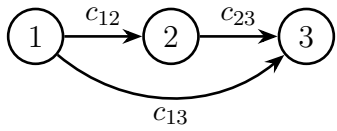
The sets of CI statements which hold in the corresponding MLBN are 
\begin{align*}
     \begin{cases}
        \ \bigl\{ 1 \indep 3 | 2 \bigr\} \ \ \ &\text{if} \ \ c_{13} \leq c_{12}c_{23} \\ 
    \ \  \ \ \ \ \emptyset \ \ \ &\text{if} \ \ \  c_{13} > c_{12}c_{23}. 
    \end{cases}
\end{align*}
The conditional independence statement in the first case arises due to the fact that if the inequality $c_{13} \leq c_{12}c_{23}$ holds, the structural equation for $X_3$ can be rewritten as 
\begin{align*}
    X_3 =& \ c_{13}X_1 \vee c_{23}X_2 \vee Z_3 \\
    =& \ c_{13}X_1 \vee c_{23}(c_{12}X_1 \vee Z_2) \vee Z_3 \\
    =& \ (c_{13} \vee c_{12}c_{23})X_1 \vee c_{23}Z_2 \vee Z_3 \\ 
    =& \  c_{12}c_{23}X_1 \vee c_{23}Z_2 \vee Z_3 \\ 
    =& \ c_{23}X_2 \vee Z_3.
\end{align*}
In particular, $X_1$ influences $X_3$ only indirectly via $X_2$.
\end{example}
Furthermore, weighted MLBNs are not faithful to the condition $\starsep$ from the previous subsection. 

\begin{example} \label{eg:21diamond}
Consider the \textit{21-diamond} with edge weights chosen such that \\ 
$c_{24} < c_{21}c_{13}c_{34}$ and corresponding structural equations:

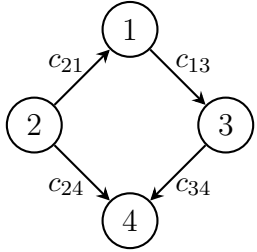
\begin{minipage}{0.45\textwidth}
    \centering
    \begin{figure}[H]
        \begin{tikzpicture}
    \begin{scope}[every node/.style={circle, draw}, node distance=1.8cm, line width=1pt, 
    every path/.style={thick, -{Stealth[length=5pt, width=5pt, scale=1]}}]
        \node (1) {$1$};
        \node (2) [below left of=1] {$2$};
        \foreach \from/\to in {2/1}
        \draw (\from) -- (\to);
        \path[every node/.style={font=\sffamily\small}]
        (1) -- (2) node [near start, left] {$c_{21}$};
        
        \node (3) [below right of=1] {$3$};
        \foreach \from/\to in {1/3}
        \draw (\from) -- (\to);
        \path[every node/.style={font=\small}]
        (1) -- (3) node [near start, right] {$c_{13}$};
        
        \node (4) [below right of=2] {$4$};
        \foreach \from/\to in {2/4,3/4}
        \draw (\from) -- (\to);
        \path[every node/.style={font=\small}]
        (2) -- (4) node [near end, left] {$c_{24}$};
        \path[every node/.style={font=\small}]
        (3) -- (4) node [near end, right] {$c_{34}$};
    \end{scope}
    \label{fig:21diamond}
\end{tikzpicture}
    \end{figure}
  
\end{minipage}
\begin{minipage}{0.45\textwidth}
\vspace{8pt}
    \begin{align*}
        X_1 &= c_{21}X_2 \vee Z_1 \\
        X_2 &= Z_2 \\
        X_3 &= c_{13}X_1 \vee Z_3 = c_{13}(c_{21}Z_2 \vee Z_1) \vee Z_3 \\ 
        X_4 &= c_{24}X_2 \vee c_{34}X_3 \vee Z_4 \\
    \end{align*}
\end{minipage}

In particular, we observe that the conditional independence statement $[X_2 \indep X_4 | X_3]$ holds in this model. However, $\CI{2,4|3}$ is not a statement in $\globalstarsep{\cg}$  as the edge $2 \rightarrow 4$ is a $\ast-$connecting path given $K = \{3 \}$.
\end{example}

Intuitively, the reason for the additional CI statement in the example above is that, due to this particular choice of edge weights, the edge $2 \rightarrow 4 $ is no longer \emph{critical}; the random variable $X_2$ influences $X_4$ only indirectly via the path $2 \rightarrow 1 \rightarrow 3 \rightarrow 4$.Phenomena such as this one motivate yet another notion of graph separation which can account for critical paths. This was introduced in \cite{amendola_conditional_2022}.
Let $(\cg,C)$ be a weighted DAG on $n$ nodes and $i,j \in V(\cg)$. We denote the set of all directed paths from $i$ to $j$ by $P(i,j)$. For a directed path $\pi = i \rightarrow l_1 \rightarrow \dots \rightarrow l_m \rightarrow j \in P(i,j)$, we define its \emph{weight} in $(\cg,C)$ as
    $\omega_C(\pi) = c_{il_1}c_{l_1l_2}\cdot \dots \cdot c_{l_mj}$.
We say that a directed path $\pi \in P(i,j)$ is \emph{critical} if its weight $\omega_C(\pi)$ attains the maximal value over all $P(i,j)$. 
The weights of the critical paths of $(\cg,C)$ are encoded in the \emph{Kleene star} of $C$. This is the matrix
\begin{align*}
    C^\ast = (c^*_{ij})_{i,j \in [n]} = \max_{\pi \in P(i, j)} \omega_C(\pi) \ , 
\end{align*}
which may be computed by taking a tropical power of $C$ (\cite{butkovic_max-linear_2010}, Proposition 1.6.15).

To define a notion of graph separation for weighted DAGs we associate a graph to each conditioning set $K$. We say that a directed path $\pi$ from $i$ to $j$ \emph{factors through K} if there exists a node $k$ along $\pi$ such that $k \in K \setminus ij$.
Let $\cg$ be a DAG on $n$ nodes and $K \subset V$. The \textit{critical DAG} of $(\cg,C)$ with respect to $K$, denoted $\cg^\ast_K(C)$ is the unweighted DAG with node set $V$ and edges determined by the condition 
    \begin{center}
    \(
    i \rightarrow j \in \cg^\ast_K(C) 
    \iff 
    \begin{aligned}[t]
        & c^\ast_{ij} > 0 \text{ and no critical directed path} \\
        &\text{ from } i \text{ to } j \text{ factors through } K.
    \end{aligned}
    \)
    \end{center}
    If $C$ is fixed and clear from context, we may write $\critdag{K}$ for readability. We say that two nodes $i$ and $j$  in $V \setminus K$ are \emph{$C^\ast-$connected} given $K$ if there is a path from $i$ to $j$ in $\critdag{K}$ of one of the five types (a) - (e) depicted in \Cref{fig:5paths}. If no such path in $\critdag{K}$ exists, we say that $i$ and $j$ are \emph{$C^\ast$-separated} given $K$, and denote this by $\Csepgiven{i}{j}{K}$.

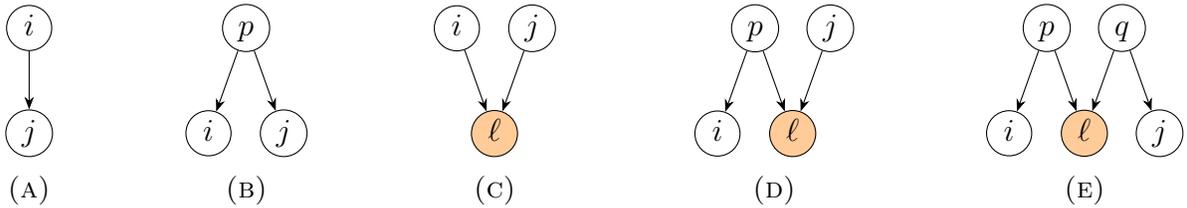
\begin{figure}[H]
    \centering
\begin{subfigure}[t]{0.10\linewidth}
\centering
\begin{tikzpicture}[inner sep=1pt]
\node[draw, circle] (i) at (0,.7) {\strut$i$};
\node[draw, circle] (j) at (0,-.7) {\strut$j$};
\draw[-{Stealth}] (i) -- (j);
\end{tikzpicture}
\caption{}
\label{fig:*-sep:a}
\end{subfigure}
\hfill
\begin{subfigure}[t]{0.15\linewidth}
\centering
\begin{tikzpicture}[inner sep=1pt]
\node[draw, circle] (i) at (-.5,-.7) {\strut$i$};
\node[draw, circle] (p) at (  0,.7) {\strut$p$};
\node[draw, circle] (j) at (+.5,-.7) {\strut$j$};
\draw[-{Stealth}] (p) -- (i);
\draw[-{Stealth}] (p) -- (j);
\end{tikzpicture}
\caption{}
\label{fig:*-sep:b}
\end{subfigure}
\hfill
\begin{subfigure}[t]{0.15\linewidth}
\centering
\begin{tikzpicture}[inner sep=1pt]
\node[draw, circle] (i) at (-.5,.7) {\strut$i$};
\node[draw, circle, fill=orange!40!white] (l) at (  0,-.7) {\strut$\ell$};
\node[draw, circle] (j) at (+.5,.7) {\strut$j$};
\draw[-{Stealth}] (i) -- (l);
\draw[-{Stealth}] (j) -- (l);
\end{tikzpicture}
\caption{}
\label{fig:*-sep:c}
\end{subfigure}
\hfill
\begin{subfigure}[t]{0.2\linewidth}
\centering
\begin{tikzpicture}[inner sep=1pt]
\node[draw, circle] (i) at (-.5,-.7) {\strut$i$};
\node[draw, circle] (p) at (  0,+.7) {\strut$p$};
\node[draw, circle, fill=orange!40!white] (l) at (+.5,-.7) {\strut$\ell$};
\node[draw, circle] (j) at (1.0,+.7) {\strut$j$};
\draw[-{Stealth}] (p) -- (i);
\draw[-{Stealth}] (p) -- (l);
\draw[-{Stealth}] (j) -- (l);
\end{tikzpicture}
\caption{}
\label{fig:*-sep:d}
\end{subfigure}
\hfill
\begin{subfigure}[t]{0.2\linewidth}
\centering
\begin{tikzpicture}[inner sep=1pt]
\node[draw, circle] (i) at (-.5,-.7) {\strut$i$};
\node[draw, circle] (p) at (  0,+.7) {\strut$p$};
\node[draw, circle, fill=orange!40!white] (l) at ( .5,-.7) {\strut$\ell$};
\node[draw, circle] (q) at (1.0,+.7) {\strut$q$};
\node[draw, circle] (j) at (1.5,-.7) {\strut$j$};
\draw[-{Stealth}] (p) -- (i);
\draw[-{Stealth}] (p) -- (l);
\draw[-{Stealth}] (q) -- (l);
\draw[-{Stealth}] (q) -- (j);
\end{tikzpicture}
\caption{}
\label{fig:*-sep:e}
\end{subfigure}
\caption{
The types of $\ast$-connecting paths between $i$ and $j$ given $K$ in a critical DAG~$\critdag{K}$. The~colored colliders $\ell$ must belong to $K$; the non-colliders $p,q$ must not belong to~$K$.}

    \label{fig:5paths}
\end{figure}
It was proven in \cite{amendola_conditional_2022} that for all but a zero-measure set of weight matrices in $\mathbb{R}^E$, weighted MLBNs are faithful to the $\Csep$ condition in the sense that
    \begin{align} 
        \Csepgiven{i}{j}{K} \ \text{holds in } (\cg,C) \iff \CI{X_i, X_j | X_K}  \ \text{holds in } X , 
    \end{align}
where $X = (X_1, \dots X_n)$ is a MLBN on $(\cg, C)$. Equivalently, the \emph{global Markov property} of a weighted DAG $(\cg, C)$ w.r.t. $\Csep$
\begin{align*} \label{def:Csepmarkov}
    \glob_{C^\ast}(\cg,C) = \bigl\{ [i \indep j |K] \ \text{s.t} \ \Csepgiven{i}{j}{K} \ \text{holds in} \ (\cg, C) \bigr\} 
\end{align*}
encodes the entire conditional independence structure of its corresponding MLBN.

\begin{proposition} \label{prop:markovinclusions}
Let $\cg$ be a DAG. For any choice of coefficient matrix $C \in \mathbb{R}^E$ it holds that 
\begin{align} 
\globaldsep{\cg} \subset \globalstarsep{\cg} \subset \globalCsep{\cg}{C}.
\end{align}
\end{proposition}
This is a reformulation of Corollary 5.9 of \cite{amendola_conditional_2022}.
\Cref{eg:cassio} and \Cref{eg:21diamond} are instances of strict inclusion in the first and second inclusions respectively; the latter fact will be shown in \Cref{eg:21diamondCsep}. 

\section{The PC algorithm in unweighted MLBNs} 
\label{unweightedPC}
The goal of this section is to prove the following statement: 

\begin{theorem} \label{prop:dstarsameoutput}
Let $\cg$ be a DAG, and $\globaldsep{\cg}$ and $\globalstarsep{\cg}$ be its global Markov properties with respect to $\dsep$ resp. $\starsep$. Applying the PC algorithm to $\globalstarsep{\cg}$ returns the same CPDAG as the one obtained upon taking $\globaldsep{\cg}$ as input.
\end{theorem}

Due to \Cref{thm:dstarmarkovclass} this is a desirable outcome, as a graph with the same skeleton and unshielded colliders as $\cg$ is a representative of the $\starsep$ Markov equivalence class of $\cg$. The first three results of this section concern the skeleton retrieval step.
We say that two nodes $i$ and $j$ are \textit{separated} with respect to a separation condition $\perp$ if $[i \perp j | K]$ holds for some $K \subset V \setminus {ij}$.
\begin{lemma} \label{lem:nonadjsep}
    Let $\cg = (V,E)$ be a DAG and $i,j \in V$. The following are equivalent: 
    \begin{enumerate}
        \item $i$ and $j$ are non-adjacent
        \item $i$ and $j$ are d-separated
        \item $i$ and $j$ are $*$-separated
    \end{enumerate}
\end{lemma}
\begin{proof}
$1 \implies 2$ is the \enquote{implies} direction of \cite[Lemma 1]{verma_equivalence_2022}, whereas $2 \implies 3$ follows directly from the definition of $\ast$-separation.
 $3 \implies 1$ is easily proven by contraposition: if $i$ and $j$ are adjacent, then the edge between them is a $*-$connecting path for any choice of $K \subset V \setminus {ij}$. 
\end{proof}
For the correctness of the skeleton retrieval step, in addition to \Cref{lem:nonadjsep} one has to show that, similarly to the d-separation setting, it suffices to consider statements in $\globalstarsep{\cg}$ where the size of $K$ is bounded in order to recover the correct skeleton. 
\begin{lemma} \label{lem:starsepgivenparents}
    Let $i$ and $j$ be non-adjacent nodes in a DAG $\cg$. Then one of the following hold: 
    \begin{center}
        $\starsepgiven{i}{j}{\pa(j)}$ \ \ \ or \ \ \ $\starsepgiven{i}{j}{\pa(i)}$.
    \end{center}
\end{lemma}
\begin{proof} 
Due to the acyclicity of $\cg$, we may assume without loss of generality that there exists no directed path from $i$ to $j$. We claim that in this case, $\starsepgiven{i}{j}{\pa(i)}$ holds. We do this by showing that no undirected path between $i$ and $j$ is $*-$connecting given $\pa(i)$. Paths between $i$ and $j$ containing no colliders are either directed paths $i \leftarrow ... \leftarrow j$ or treks of the form $i \leftarrow ... \leftarrow \rightarrow ... \rightarrow j$. Both of these types of path are blocked by $\pa(i)$. Paths between $i$ and $j$ which contain exactly one collider and are not blocked by $\pa(i)$ are of the form $i \rightarrow ... \rightarrow k_1\rightarrow k \leftarrow k_2- ... - j$, where the portion of the path from $k_2$ to $j$ is either directed or a trek. For this path to be $*$-connecting given $\pa(i)$, the middle node of the collider $k$ would have to be an ancestor of a node in $\pa(i)$, and thus also of $i$. However, this would imply that $\cg$ contains the directed cycle $i \rightarrow ... \rightarrow k \rightarrow ... \rightarrow i$, contradicting the assumption that $\cg$ is acyclic.
\end{proof}
\begin{proposition} \label{prop:dstarskeleton}
The output of \Cref{alg:PCalgskel} does not change upon replacing $\globaldsep{\cg}$ with $\globalstarsep{\cg}$  as input.  
\end{proposition}
\begin{proof}
Any two non-adjacent nodes $i,j$ in $G= \skel(\cg)$ are separated by a conditioning set of cardinality at most $\max\{|\pa(i)|, |\pa(j)|\}$ due to a combination of \Cref{lem:nonadjsep} and \Cref{lem:starsepgivenparents}. The fact that \Cref{lem:nonadjsep} is an equivalence means that no additional edges are cut when applying \Cref{alg:PCalgskel} to $\globalstarsep{\cg}$. Hence, the output is once again that of \Cref{alg:PCalgskel}, which is the undirected skeleton $G$.       
\end{proof}

Finally, we need to show that the edge orientation step behaves the same in both settings. 
\begin{lemma} \label{lem:starsepcolliders}
The edge orientation step of PC does not change upon replacing $\globaldsep{\cg}$ with $\globalstarsep{\cg}$ as input. 
\end{lemma}
\begin{proof}
The triples oriented by $\mathcal{R}1$  of \Cref{misc:orientationrules} depend entirely on conditional independence statements with $|K| =1$. Subsequent edge orientations according to the rules $\mathcal{R}2-\mathcal{R}4$ are not dependent on the statements taken as input. Thus, the claim follows from the fact that if $|K| \leq 1$, then  $\dsepgiven{i}{j}{K} \iff \starsepgiven{i}{j}{K}$ holds. 
\end{proof}
The proof of \Cref{prop:dstarsameoutput} now follows by combining \Cref{prop:dstarskeleton} and \Cref{lem:starsepcolliders}.

\section{The PCstar Algorithm for weighted MLBNs}
\label{sec:pcstar}
In this section we discuss the behavior of the PC algorithm when the input is a set of CI statements $\globalCsep{\cg}{C}$ arising from a weighted MLBN. In this setting, the additional CI statements from \eqref{def:Csepmarkov} lead to an output of \Cref{alg:PCalgskel} which may no longer be the undirected skeleton of $\cg$.
\begin{example} \label{eg:21diamondCsep}

Assume that the true DAG $(\cg, C)$ underlying a MLBN is a $21-$diamond  as depicted in \Cref{fig:21diamond} with coefficient matrix $C$ chosen such that $c_{24} < c_{21}c_{13}c_{34}$. The set $\globalCsep{\cg}{C}$ is
    \begin{align*}
        1 \indep 4 | \{3\} \ &, \ 1 \indep 4 | \{ 2,3 \} \\
        2 \indep 3 | \{1\} \ &,  \ 2 \indep 3 | \{ 1, 4 \} \\
        2 \indep 4 | \{1\} \ &, \  2 \indep 4 | \{3\} \  ,  \   2 \indep 4 | \{ 1 , 3 \}.  
    \end{align*}

Applying \Cref{alg:PCalgskel} to $\globalCsep{\cg}{C}$ outputs the following subgraph of $K_4$:

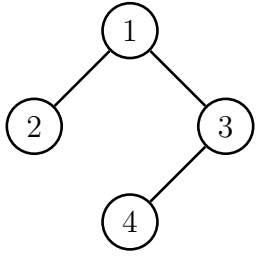
\begin{figure}[H]
    \centering
    \begin{tikzpicture}
    \begin{scope}[every node/.style={circle, draw}, node distance=1.8cm, line width=1pt]
        \node (1) {$1$};
        \node (2) [below left of=1] {$2$};
        \node (3) [below right of=1] {$3$};
        \node (4) [below right of=2] {$4$};
        
        \draw (1) -- (2);
        \draw (1) -- (3);
        \draw (3) -- (4);
    \end{scope}
\end{tikzpicture}

    \caption{The output of the skeleton retrieval step of PC on the $21$-diamond from \Cref{fig:21diamond} is the complete undirected graph with the edges $\{1,4\}$,$\{2,3\}$,$\{2,4\}$ cut. }
    \label{fig:21diamondskel}
\end{figure}

Unlike the $\dsep$ and $\starsep$ cases, this is a proper subgraph of the undirected skeleton of $\cg$.

\end{example}

To make sense of this phenomenon, we introduce the notion of the \emph{weighted transitive reduction} of a weighted DAG $(\cg,C)$. 

\subsection{Skeleton retrieval and the weighted transitive reduction}

\begin{definition} \label{def:wtr}
    Let $(\cg,C)$ be a weighted DAG on $n$ nodes. The \emph{weighted transitive reduction} of $(\cg, C)$ is the weighted DAG $(\cgtr, C^{tr})$ on $n$ nodes with weighted edges determined as follows: 
    \begin{align*}
        i \rightarrow j  \in \cgtr \ \text{with weight} \ c_{ij}  \ :\iff \ \text{\parbox[c][2.5em][c]{5.5cm}{\centering The edge $i \rightarrow j $ is the unique critical path  from $i$ to $j$ in $\cg$.}}
    \end{align*}
\end{definition}

Clearly, $\cgtr$ is a subgraph of $\cg$. Versions of it appear in \cite{gissibl_max-linear_2018} (under the name \emph{minimum max-linear DAG}) and in \cite{amendola_tropical_2024} where it is used in the combinatorial classification of polytropes. 
The results from this section will show that it is the sparsest subgraph of $\cg$ which is capable of preserving its global Markov property. 

\begin{definition} \label{def:genericC}
For a fixed DAG $\cg$ with edge set $E$, a matrix $C \in \mathbb{R}^E$ is \emph{generic} if it is such that for all $i,j \in V$, the critical $i-j$ path in $(\cg,C)$ is unique if it exists.
    
\end{definition}

The term generic is justified by the fact that the set of all matrices supported on a fixed DAG $\cg$  which give rise to multiple critical $i-j$ paths form a zero-measure set in  $\mathbb{R}_{>0}^E$. Recent work of Boege et al. \cite{boege_polyhedral_2025} associates each DAG $\cg$ to a polyhedral fan in the space of weight matrices $\mathbb{R}_{>0}^E$. In this fan, these zero-measure sets are precisely the faces of codimension $\geq 1$.
\begin{remark} \label{rmk:wtrminimal}
For generic $C$, $\cgtr$ is the unique minimal subgraph which has the same critical paths as $\cg$. This is because in any critical path $i \rightarrow \dots k_1 \rightarrow k_2  \dots \rightarrow j$ in $\cg$, each intermediate edge $k_1 \rightarrow k_2$ is the unique critical $k_1 - k_2$ path in $\cg$, and the property of being a critical path is transitive. 
\end{remark}
Two graphs with equal unique critical paths for each pair $(i,j)$ give rise to the same global Markov properties. For a two nodes $i,j$ in a weighted DAG $(\cg,C)$ with generic $C$, we denote the sequence of nodes $(i, k_1, k_2 , \dots , j)$ corresponding to the unique critical $i-j$ path by $\picrit{i}{j}(\cg,C)$.
\begin{lemma}\cite[Lemma 2.4]{boege_polyhedral_2025} \label{lem:critpathsmarkov}
        Let $(\cg, C)$ and $(\cg', C')$ be two weighted DAGs on $n$ nodes with generic weights $C$ and $C'$. Then 
    \[
        \globalCsep{\cg}{C} = \globalCsep{\cg'}{C'} \iff \picrit{i}{j}(\cg, C) = \picrit{i}{j}(\cg', C') \,\text{ for all $i \neq j$}.
    \]
\end{lemma}

\begin{corollary} \label{cor:wtrsparsest}
For generic $(\cg, C)$, $\cgtr$ is the unique sparsest subgraph of $\cg$ capable of encoding $\globalCsep{\cg}{C}$, in the sense that it fulfills the following condition:
\begin{align} \label{eq:encodecon}
    \text{There exists a $C'$ supported on $\cgtr$ such that $\globalCsep{\cg}{C} = \globalCsep{\cgtr}{C'}$.} 
\end{align}
\end{corollary}
\begin{proof}
    The equality $\globalCsep{\cg}{C} = \globalCsep{\cgtr}{C^{tr}}$ follows from the previous lemma and \Cref{rmk:wtrminimal}. It is evident from \Cref{def:wtr} that deleting any edge of $\cgtr$ would change its critical paths, thus implying that it is the sparsest subgraph of $\cg$ fulfilling \eqref{eq:encodecon}.
\end{proof}

The following result characterizes the behavior of the skeleton retrieval step of PC in the weighted setting.
\begin{lemma} \label{lem:Csepgivenparents}
    Let $(\cg, C)$ be a weighted DAG. For $i,j \in V$ it holds that 
    \begin{align*}
        \Csepgiven{i}{j}{\pa(j)} \ \text{in $(\cg, C)$} \iff \starsepgiven{i}{j}{\pa(j)} \ \text{in $\cgtr$}.
    \end{align*}
\end{lemma}
\begin{proof}
We assume that there exists a directed path from $i$ to $j$ (In all other cases the equivalence is trivial). \\
\enquote{$\implies$}: If $\Csepgiven{i}{j}{\pa(j)}$ holds in $\cg$ then in particular $i \rightarrow j$ is not an edge of $\critdag{\pa(j)}$ (There is no $*$-connecting path of type (a) in this critical DAG). Thus, there exists a critical path from $i$ to $j$ which factors through $\pa(j)$. This implies $i \rightarrow j \notin \cgtr$, and therefore (by \Cref{lem:nonadjsep}) $i$ and $j$ are $*-$ separated by $\pa(j)$ in $\cgtr$. \\ 
\enquote{$\impliedby$}: If $\starsepgiven{i}{j}{\pa(j)}$ holds in $\cgtr$, then this implies $\Csepgiven{i}{j}{\pa(j)}$ holds in $\cgtr$ by \Cref{prop:markovinclusions}. Due to \Cref{lem:critpathsmarkov}, the statement also holds in $(\cg,C)$. 
\end{proof}

\begin{theorem} \label{thmPCwtr}

Let $(\cg, C)$ be a weighted DAG. Applying \Cref{alg:PCalgskel} to $\globalCsep{\cg}{C}$ retrieves $\skel(\cgtr)$.

\end{theorem}

\begin{proof}

Let $i, j \in V$ and assume without loss of generality that no directed path from $j$ to $i$ exists. Using \Cref{lem:nonadjsep} and the definition of the global Markov property, \Cref{lem:Csepgivenparents} can be reformulated as follows: 
\begin{align*}
    i \ \text{and} \ j \ \text{are non-adjacent in}  \ \cgtr \iff& \ \Csepgiven{i}{j}{\pa(j)} \ \text{holds in } \ (\cg,C) \\
    \iff& \ \CI{i,j|K} \ \in \globalCsep{\cg}{C} 
    \end{align*} 
Therefore, the procedure of starting with the complete undirected graph and removing edges whenever two nodes are $C^\ast$-separated recovers $\skel(\cgtr)$. Furthermore (similarly to the d- and unweighted $\ast-$-separation cases) it suffices to query only the separation statements with conditioning set of cardinality up to $\max_{v \in V}\indeg(v)$ due to \Cref{lem:Csepgivenparents}.
\end{proof}

\subsection{Collider orientation}

Applying \Cref{alg:PCalgskel} to $\globalCsep{\cg}{C}$ yields the skeleton of a well-defined subgraph of $(\cg,C)$ which contains all of its critical paths. As $G^{tr} := \skel(\cgtr)$ is generally a proper subgraph of the skeleton of $\cg$, it may not contain all of its unshielded triples. Nevertheless, given $\globalCsep{\cg}{C}$ one may expect to be able to orient some of the edges of $G^{tr}$ using the rules in \Cref{misc:orientationrules}, starting with $\mathcal{R}1$. The following example illustrates that a direct application of $\mathcal{R}1$ can lead to an incorrect output.

\begin{example} \label{eg:diamond}

 Consider the weighted DAG $(\cg, C)$ depicted below
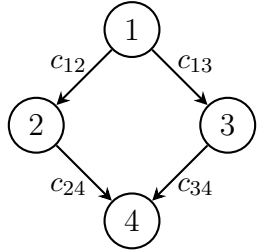
\begin{figure}[H]
\centering
\begin{tikzpicture}
    \begin{scope}[->,every node/.style={circle,draw},line width=1pt, node distance=1.8cm, every path/.style={thick, -{Stealth[length=5pt, width=5pt, scale=1]}}]
    \node (1) {$1$};
    \node (2) [below left of=1] {$2$};
    \foreach \from/\to in {1/2}
    \draw (\from) -- (\to);
    \path[every node/.style={font=\sffamily\small}]
    (1) -- (2) node [near start, left] {$c_{12}$};
    \node (3) [below right of=1] {$3$};
    \foreach \from/\to in {1/3}
    \draw (\from) -- (\to);
    \path[every node/.style={font=\small}]
    (1) -- (3) node [near start, right] {$c_{13}$};
    \node (4) [below right of=2] {$4$};
    \foreach \from/\to in {2/4,3/4}
    \draw (\from) -- (\to);
    \path[every node/.style={font=\small}]
    (2) -- (4) node [near end, left] {$c_{24}$};
    \path[every node/.style={font=\small}]
    (3) -- (4) node [near end, right] {$c_{34}$};
    \end{scope}
\end{tikzpicture}
\vspace{-1cm}
\label{fig:diamond}
\end{figure}
with edges weighted such that the path $1-3-4$ is the unique critical $1-4$ path (i.e. $c_{12}c_{24} < c_{13}c_{34}$). The global Markov property is
\begin{align*}
    \globalCsep{\cg}{C} = \bigl\{ [1 \indep 4  \ | 3] \ , \  [1 \indep 4  \ | 23] \ , \ [2 \indep 3 \ | 1] \bigr\}.
\end{align*}

Applying the skeleton retrieval of PC to this set cuts the edges $\{1,4\}$ and $\{ 2,3 \}$ from the complete graph on 4 nodes, yielding precisely $G = \skel(\cg)$. This is consistent with \Cref{thmPCwtr}, as $\cg = \cgtr$ holds for this DAG. However, the subsequent orientation according to the rules in \Cref{misc:orientationrules} leads to an incorrect output. For the pair $(1,4)$, \Cref{alg:PCalgskel} sets $Sepset(1,4) = \{ 3\}$, and applying
$\mathcal{R}1$ to the unshielded triple $1 - 2 - 4$ would orient it as $1 \rightarrow 2 \leftarrow 4$ due to the fact that  $2 \notin Sepset(1,4)$.   
\end{example}

 From the point of view of graph separation, the \enquote{error} in the previous example occurs because the critical DAG $\cg^\ast_{\{2\}}$ contains the edge $1 \rightarrow 4$, which is a $\ast$-connecting path of type (a). However, $1$ and $4$ are not $\ast$-connected in $\cg^*_{\{2,3\}}$ and indeed one has that $\CI{1,4|23} \in \globalCsep{\cg}{C}$. So while the node $2$ does not lie in the minimal separating set for $1$ and $4$, $2$ does lie in \emph{some} separating set. This suggests that for collider orientation it may be necessary to look beyond the first separating set encountered. The following proposition states that non-minimal separating sets of a certain form provide sufficient information to orient all colliders.
 
\begin{proposition} \label{prop:findcoll}
Let $(\cg, C)$ be a weighted DAG and $\{i, k,j \}$ be an unshielded triple in $\skel(\cgtr)$. Then the following holds: 

\begin{center}
    $i \rightarrow k \leftarrow j$ is a collider in $\cgtr$ $\iff \CI{i,j| k \cup \pa(j) \cup \pa(i)} \notin \globalCsep{\cg}{C}$.  
\end{center}
\end{proposition}
\begin{proof}
The direction \enquote{$\implies$} is straightforward: If $i \rightarrow k \leftarrow j$ is the orientation of the triple $\{ i, k,j\}$ in the true DAG $(\cg,C)$, then this is a $\ast$-connecting path in $\cg^\ast_K$ of type (c) for any conditioning set $K$ with $k \in K$. In particular, $\CI{i,j|k \cup \pa(j) \cup \pa(i)} \notin \globalCsep{\cg}{C}$. \\
    We prove \enquote{$\impliedby$} by contraposition: if $\{i, k, j \}$ is an unshielded triple in $\cgtr$ which is not oriented $i \rightarrow k \leftarrow j$, then conditioning on $K:= k \cup \pa(i) \cup \pa(j)$ excludes the possibility of $\ast$-connecting paths in $\cg^{tr^\ast}_{C_K}$ between $i$ and $j$ of any type. The assumption directly implies that no path of type (c) exists. All paths of type (b), (d), (e) are blocked by elements of $\pa(i) \cup \pa(j)$. No $i-j$ path of type (a) exists, as this would imply either $i\rightarrow j \in \cgtr$ or $j \rightarrow i \in \cgtr$, contradicting the assumption that $\{i,j,k\}$ is an unshielded triple in $\cgtr$.
\end{proof}

\begin{remark} \label{rmk:collidercomplexity}
     The process of collecting non-minimal separating sets for each $i,j$ of the form above remains polynomial in the maximal in-degree of $\cg$. This is important for the total complexity of the modified PC algorithm which will be discussed at the end of this section. 
\end{remark}

 With the slight modifications described above, given $\globalCsep{\cg}{C}$, one can detect and correctly orient all of the unshielded colliders of $\cgtr$. After subsequently orienting edges according to the rules $\mathcal{R}2- \mathcal{R}4$ (recall that these are not dependent on the input statements), one obtains an approximation of $\cgtr$ which is \enquote{at least as good} as the output of the PC algorithm in the d-separation setting, in the sense that the skeleton and unshielded colliders of a graph capable of encoding the same conditional independence information as the true DAG is recovered. Due to the definition of $\Csep$ (in particular its reliance on \emph{directed} critical paths), additional edges may be oriented. 
\subsection{Orienting induced cycles}
In this section we show how the additional statements which arise from $\Csep$ may be used to orient additional edges in the CPDAG which are not orientable in either the $\dsep$ or $\starsep$ paradigm. We begin with the following example which shows that graphs which are Markov equivalent with respect to $\equivast / \equivd$ may have differing $\Csep$ global Markov properties. 

\begin{example} \label{eg:diamonds}
The diamond graph from \Cref{eg:diamond} and the 21-diamond \Cref{eg:21diamond} are Markov equivalent w.r.t. $\equivd / \equivstar$, as they have the same undirected skeleton and both contain the unshielded collider $2\rightarrow4 \leftarrow3$. However, we have seen in  \Cref{eg:diamond} and \Cref{eg:21diamondCsep} that there exist choices of edge weights such that their $\Csep$ global Markov properties differ. (In fact, the $\Csep$ global Markov properties of these two graphs differ for any choice of edge weights by \Cref{lem:critpathsmarkov}.) At the same time, their $\Csep$ Markov properties encode information about their respective critical paths. For example, in the diamond from \Cref{eg:diamond} the statement $[1\indep 4|3]$ holds because $1\rightarrow3\rightarrow4$ is the unique critical $1-4$ path, whereas the fact that $2\rightarrow1\rightarrow3\rightarrow4$ is the unique critical $2-4$ path in the 21-diamond from \Cref{eg:21diamondCsep} gives rise to the statement $[2\indep 4 | 13]$.
\end{example}
Differences such as those highlighted in \Cref{eg:diamonds}  allow \emph{induced cycles} of $\cgtr$ of a certain kind to be oriented. 
\begin{definition}
 Let $\cg = (V,E)$ be a graph. For $W \subset V$, the \textit{induced subgraph} of $\cg$ w.r.t $W$ is the subgraph $\cg[W]$ with node set $W$ containing the edge $ i \rightarrow j \in E $ if and only if $i, j \in W$. If $\cg[W]$ is a simple cycle, we call it an \textit{induced cycle} of $\cg$. 
\end{definition}
Observe that an induced cycle of a DAG containing exactly one collider has a unique \emph{source} (i.e. a node with induced in-degree zero), which determines the orientation of all other edges. Under certain additional assumptions, the sources of cycles of this type in $\cgtr$ may be located by querying $\globalCsep{\cg}{C}$. 
\begin{lemma} \label{lem:cycleori}
    For a weighted DAG $(\cg,C)$ with generic $C$, let $\mathcal{H} = \cgtr[W]$ be an induced cycle of $\cgtr$ with the following two properties: 
    \begin{enumerate}[label = (\roman*)] \label{misc:orientabilityconditions}
        \item $\ch$ contains a unique unshielded collider $k_1\rightarrow k \leftarrow k_2$,
        \item For any $w \in W \setminus k$, the unique critical $w-k$ path is contained entirely in $W$.
    \end{enumerate}
    For $i \in W \setminus \{ k_1, k, k_2 \}$ the following holds:  

\begin{align} \label{sourceset}
        i \  \text{is the source of} \  \mathcal{H} \iff \#\{ j \in W \ |& \exists \ K \subset V \setminus W  \ \text{s.t} \  \Csepgiven{i}{k}{jK} \}>0 \  \\ 
        &\text{is maximal in $W \setminus \{ k_1, k, k_2 \}$.} \nonumber 
    \end{align}
    Moreover, the maximum is the length of the critical path from the source to the sink $k$.
\end{lemma}
\begin{proof}
We start by considering the case in which $\cgtr$ itself is a cycle. In this setting, the statements appearing in the set on the right hand side of \eqref{sourceset} have conditioning sets of cardinality one. 
If $i \in V \setminus \{ k, k_1, k_2 \}$ is the source then there exist two directed $i -k$ paths
\begin{align*}
    \pi_1 \ : \ i = j_{1,0} \rightarrow j_{1,1} \rightarrow j_{1,2} \rightarrow \cdots \rightarrow j_{1,m_1} &= k_1 \rightarrow k \\
       \pi_2 \ : \ i = j_{2,0} \rightarrow j_{2,1} \rightarrow j_{2,2} \rightarrow \cdots \rightarrow j_{2,m_2} &= k_2 \rightarrow k. 
\end{align*}
Assume that $\pi_1$ is the (by genericity of $C$ unique) critical directed $i-k$ path.
 
Then there exist $m_1$ statements which separate $i$ and $k$ of the form $\Csepgiven{i}{k}{j_{1,\ell}}$ for $\ell \in \{ 1, \dots m_1\} $. These are the only separation statements of cardinality one separating $i$ and $k$, as for any intermediate $j'$ along $\pi_2$, the critical DAG $\critdag{j'}$ contains the edge $i\rightarrow k$ arising from the unblocked critical path $\pi_1$. This edge is a $\ast-$connecting path of type (a). Any other $j_{1,\ell}$ along $\pi_1$ will be separated from $k$ only given any one of the  intermediate nodes $j_{1, \ell+1} \dots j_{1, m_1}$. This is because for any $q \in \{\ell+1, \dots m_1-1\}$, and $p \in \{0, \dots, \ell \}$

the critical DAG $\critdag{j_{1,q}}$  does not contain the edge $ j_{1,p} \rightarrow k$, meaning that no $\ast-$connecting path between $j_{1, \ell}$ and $k$  of type (a) or (b) is present. The possibility of $\ast-$connecting paths of types (c) - (e) is excluded by the fact that the underlying DAG contains only one collider at $k$. Moreover, none of the nodes along the non-critical path of the form $j_{2,\ell}$ , for $\ell \in \{1, \dots   , m_2-1\}$ are separated from $k$ upon conditioning on one node, as no single node simultaneously blocks $\pi_1$ and the directed $i-j_{2,\ell}$ portion of the cycle. More specifically, for any $j \in V \setminus\{ k_1, k, k_2, j_{2,\ell}\}$, the critical DAG $\critdag{j}$ contains the edge $j_{2,\ell} \rightarrow k$ in the case that $j \not \in \{j_{2, \ell+1} , \dots j_{2,m_2} \}$ and the 3-node trek $j_{2,\ell}\leftarrow i \rightarrow k$ otherwise. In either case, $j_{2,\ell}$ and $k$ are $\ast-$connected given $j$. This completes the proof in the case that $\cgtr$ itself is a cycle. 

To generalize the proof for induced cycles it suffices to observe that given a weighted transitive reduction $\cgtr$ and an induced cycle $\mathcal{H} = \cgtr[W]$ the following holds: 
\begin{align} \label{eqn:equivinducedcycle}
    \Csepgiven{i}{k}{j} \ \text{holds in} \ \mathcal{H} \iff \Csepgiven{i}{k}{j \cup \an(W)}  \ \text{holds in}  \ \cgtr. 
\end{align}
This is because by adding $\an(W)$ to the conditioning set, all treks from $W \setminus k$ to $k$ containing intermediate nodes outside of $W$ are blocked. Furthermore, the second condition of \Cref{misc:orientabilityconditions} implies that any directed path between nodes in $W$ containing intermediate nodes outside of $W$ does not contribute to the critical path structure of $\mathcal{H}$. Taken together, these last two facts imply that the critical DAG $\mathcal{H}^\ast_{j}$ is precisely the restriction of $\cg^{tr\ast}_{C_{j \cup \an(W)}}$ to $W$, which in turn implies \eqref{eqn:equivinducedcycle}. This completes the proof. 
\end{proof}
We call induced cycles which fulfill the two of conditions \eqref{misc:orientabilityconditions} of \Cref{lem:cycleori} \textit{orientable}. The second condition is generally not verifiable given only the global Markov property. However, a sufficient condition for orientability can be verified given the information obtained in the skeleton retrieval and collider orientation steps. 

\begin{figure}
\begin{minipage}{0.35\textwidth}
    \centering
    \scalebox{0.8}{
\begin{tikzpicture}[
    ->,
    >=stealth',
    shorten >= 1pt,
    node distance = 2.8cm,
    thick,
    every node/.style={circle, draw, line width=1pt},
    every path/.style={thick, -{Stealth[length=5pt, width=5pt, scale=1]}}
]
        \node[minimum size = 8mm] (A) at (0,2) {$i$};
        \node (B) at (1.99,0.99) {\phantom{}};
        \node (B2) at (1.15, 1.49) {\phantom{}};
        \node[minimum size = 8mm] (J) at (1.99, -0.3) {$j$};
        \node (C) at (1.99,-1.49) {\phantom{}};
        \node[minimum size = 8mm] (D) at (0,-2.99) {$k$};
        \node (E) at (-1.99,-1.49) {\phantom{}};
        \node (F) at (-1.99,0.99) {\phantom{}};

        \draw (A) -- (B2);
        \draw (B2) -- (B);
        \draw (A) -- (F);
        \draw (B) -- (J);
        \draw (J) -- (C);
        
        \draw (F) -- (E);
        \draw (E) -- (D);
        \draw (C) -- (D);
    
    \end{tikzpicture}
}    
    \vspace{-0.7cm}
    $\cg$
\end{minipage}
\hspace{1cm}
\begin{minipage}{0.35\textwidth}
    \centering
    \scalebox{0.8}{
\begin{tikzpicture}[
    ->,
    >=stealth',
    shorten >= 1pt,
    node distance = 2.8cm,
    thick,
    every node/.style={circle, draw, line width=1pt},
    every path/.style={thick, -{Stealth[length=5pt, width=5pt, scale=1]}}
]
        \node[minimum size = 8mm] (A) at (0,2) {$i$};
        \node (B) at (1.99,0.99) {\phantom{}};
        \node (B2) at (1.15, 1.49) {\phantom{}};
        \node[minimum size = 8mm, fill = orange!50] (J) at (1.99, -0.3) {$j$};
        \node (C) at (1.99,-1.49) {\phantom{}};
        \node[minimum size = 8mm] (D) at (0,-2.99) {$k$};
        \node (E) at (-1.99,-1.49) {\phantom{}};
        \node (F) at (-1.99,0.99) {\phantom{}};

        \draw (A) -- (B2);
        \draw (B2) -- (B);
        \draw (A) -- (F);
        \draw (B) -- (J);
        \draw (J) -- (C);
        
        \draw (F) -- (E);
        \draw (E) -- (D);
        \draw (C) -- (D);
        \draw [->] (A) to [bend right = 40 ](B);
        \draw [->] (A) to (E); 
        \draw [->] (F) to [bend left = 40] (D);
        \draw [->] (J) to  (D);
    
    \end{tikzpicture}
}
    \vspace{-0.7cm}
    $\critdag{j}$

\end{minipage}

\caption{Left: an orientable cycle  $\cg$. Right: the critical DAG $\critdag{j}$, if the weights are chosen such that the directed path through $j$ is the unique $i-k$ critical path. }
\end{figure}
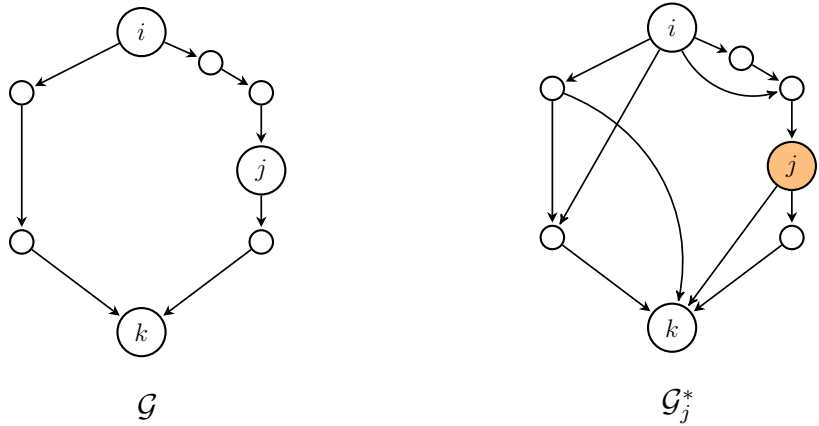

\begin{proposition}
    
\label{rmk:orientablecond}
    
Let $\cgtr$ be the weighted transitive reduction of some true DAG $(\cg,C)$ with $C$ generic, and $\mathcal{H} = \cgtr[W]$ be an induced cycle of $\cgtr$ containing exactly one unshielded collider. $\mathcal{H}$ is orientable if it fulfills the following condition:
\begin{center}
    For all $w \in W \setminus k$ any undirected $w-k$ path containing nodes \\  outside of $W$ contains an unshielded collider in $\cgtr$.
\end{center}
\end{proposition}
\begin{proof}
One has to show that the condition in \Cref{rmk:orientablecond} implies the second statement of \eqref{misc:orientabilityconditions}. This follows immediately from the observation that for all $w \in W \setminus k$ if any $w-k$ path containing nodes outside of $W$ contains a collider triple, then no trek or directed path from $w$ to $k$ exists apart from those which factor entirely through $W$.
\end{proof}

 While \Cref{rmk:orientablecond} is a stronger condition on $\cgtr$ than the second statement of \eqref{misc:orientabilityconditions}, it still allows for the orientability of cycles and certain kinds of \enquote{fat path} graphs \cite{fink_sullivant_2016}, as well as level-1 phylogenetic networks \cite{steel2016phylogeny}.

\begin{remark} \label{rmk:srcneighbor}
Using an argument similar to that of \Cref{lem:cycleori}, the set on the right hand side of \eqref{sourceset} is empty for all $i \in W \setminus \{k_1, k_2, k \}$ if and only if the source of the induced cycle is either $k_1$ or $k_2$. However, it is impossible to distinguish between these two cases by querying $\globalCsep{\cg}{C}$; if no $i \in W \setminus \{k_1, k_2, k \}$ is separated from $k$ upon conditioning on only one other node $j$ in the cycle, then there exist two $\ast-$connecting paths from $i$ to $k$ in $\mathcal{H}^\ast_{j}$, which due to our assumptions must correspond to paths of type (a) and (b) through $k_1$ and $k_2$. Thus, if we encounter such a cycle, we can do no better than to \enquote{mark} $k_1$ and $k_2$ as possible sources.

\end{remark}

The arguments above rely on the fact that in a cycle containing a unique collider $k$, each node $i \in W \setminus \{k \}$ is connected to $k$ either by two directed paths (if $i$ is the source) or one directed path and one trek through the source (otherwise). Thus, they do not generalize to the case of induced cycles containing more than one collider.
\begin{example} \label{eg:hexagons}
Assume that, given some true weighted DAG $(\cg,C)$, the skeleton of its weighted transitive reduction (i.e. the output of \Cref{alg:PCalgskel} applied to $\globalCsep{\cg}{C}$) is a hexagon. Due to acyclicity, $\cgtr$ must contain at least one collider. These colliders may be found with the condition from \Cref{prop:findcoll}. Consider first the case that there is a unique collider, w.l.o.g $6 \rightarrow 5 \leftarrow 4$. 
\begin{figure}
\begin{minipage}{0.35\textwidth}

    \centering
        \begin{tikzpicture}[
    scale = 1,
    every path/.style={thick, -{Stealth[length=8pt, width=8pt]}}
]
        \path (0,1.49) coordinate (A); 
        \path (0.99,0.99) coordinate (B);
        \path (0.99,-0.49) coordinate (C);
        \path (0,-0.99) coordinate (D); 
        \path (-0.99,-0.49) coordinate (E);
        \path (-0.99,0.99) coordinate (F);
    
        \draw (A) -- (B);
        \draw (A) -- (F);
        \draw (B) -- (C);
        \draw (F) -- (E);
        \draw (E) -- (D);
        \draw (C) -- (D);
        \path (3,1.49) coordinate (A2); 
        \path (3.99,0.99) coordinate (B2);
        \path (3.99,-0.49) coordinate (C2);
        \path (3,-0.99) coordinate (D2); 
        \path (1.99,-0.49) coordinate (E2);
        \path (1.99,0.99) coordinate (F2);
    
        \draw (A2) -- (B2);
        \draw (F2) -- (A2);
        \draw (B2) -- (C2);
        \draw (F2) -- (E2);
        \draw (E2) -- (D2);
        \draw (C2) -- (D2);
        \node at (1.5, 0.4) {\Large$\not\approx$};
    
    \end{tikzpicture}    
    Detectable with \Cref{lem:cycleori}
\end{minipage}
\hfill
\begin{minipage}{0.55\textwidth}
    \centering
           \begin{tikzpicture}[
    scale = 1,
    every path/.style={thick, -{Stealth[length=8pt, width=8pt]}}
]
        \path (0-5,1.49) coordinate (A); 
        \path (0.99-5,0.99) coordinate (B);
        \path (0.99-5,-0.49) coordinate (C);
        \path (0-5,-0.99) coordinate (D); 
        \path (-0.99-5,-0.49) coordinate (E);
        \path (-0.99-5,0.99) coordinate (F);
    
        \draw (B) -- (A);
        \draw (A) -- (F);
        \draw (C) -- (B);
        \draw (F) -- (E);
        \draw (E) -- (D);
        \draw (C) -- (D);
        \path (3-5,1.49) coordinate (A2); 
        \path (3.99-5,0.99) coordinate (B2);
        \path (3.99-5,-0.49) coordinate (C2);
        \path (3-5,-0.99) coordinate (D2); 
        \path (1.99-5,-0.49) coordinate (E2);
        \path (1.99-5,0.99) coordinate (F2);
    
        \draw (A2) -- (B2);
        \draw (F2) -- (A2);
        \draw (B2) -- (C2);
        \draw (E2) -- (F2);
        \draw (E2) -- (D2);
        \draw (C2) -- (D2);
        \node at (1.5-5, 0.4) {\Large$ \approx$};
  
        \path (1.49 -1,-1.49+3) coordinate (A3); 
        \path (2.49-1,-1.99+3) coordinate (B3);
        \path (2.49-1,-3.49+3) coordinate (C3);
        \path (1.49-1,-3.99+3) coordinate (D3); 
        \path (0.49-1,-3.49+3) coordinate (E3);
        \path (0.49-1,-1.99+3) coordinate (F3);
    
        \draw (B3) -- (A3);
        \draw (F3) -- (A3);
        \draw (B3) -- (C3);
        \draw (E3) -- (F3);
        \draw (E3) -- (D3);
        \draw (D3) -- (C3);
        \draw (E3) -- (C3);
    \end{tikzpicture}
    Not detectable with \Cref{lem:cycleori}
\end{minipage}
\caption{The two cycles on the left have differing global Markov properties. Of the cycles on the right, the first two have equal global Markov properties, meaning that they are indistinguishable at the level of CI data. The cycles of the rightmost graph are not orientable as they are either not induced, or do not contain exactly one unshielded collider.}
\end{figure}

If the source of the cycle is at $6$ or $4$, then no node in $\cgtr$ can be separated from the sink $5$ by conditioning on one node. In accordance with \Cref{rmk:srcneighbor}, we cannot distinguish from these two cases by querying $\globalCsep{\cg}{C}$. If the source of the cycle is any other node, then we may find it by the condition in \Cref{lem:cycleori}.

If the hexagon contains two colliders, for example $6 \rightarrow 5 \leftarrow 4$ and $1 \rightarrow 2 \leftarrow 3$, then the remaining undirected edges $1-6$ and $3-4$ cannot be oriented. In fact, in this example, each of the four possibilities of orienting these two edges give rise to the same $\Csep$ global Markov property.

\end{example}

We conclude this section by introducing the \texttt{PCstar} algorithm, which utilizes the previous results to reconstruct a partially oriented weighted transitive reduction of a true unknown graph given its $C^\ast$-global Markov property. \bigbreak 
\begin{algorithm}[H] 
\caption{\texttt{PCstar}}
\label{alg:PCstar}
\Input{A complete set of CI statements $\globalCsep{\cg}{C}$ coming from a graphical model faithful to $C^\ast-$separation}  \Output{A CPDAG approximating the sparsest graph with the same Markov property as $(\cg,C)$ for an appropriate choice of weights}

Recover $G^{tr} = \skel(\cgtr)$ by applying \Cref{alg:PCalgskel} to $\globalCsep{\cg}{C}$. \\
Detect and orient the unshielded colliders of $G^{tr}$ using \Cref{prop:findcoll}. \\
Optional: \For{each orientable induced cycle $ \cgtr[W] = \mathcal{H}$}{
     \phantom{...................} Orient $\mathcal{H}$ using \Cref{lem:cycleori} 
}
Orient the remaining edges of $G^{tr}$ according to $\mathcal{R}2-\mathcal{R}4$ of \Cref{misc:orientationrules}

\end{algorithm}
\begin{theorem} \label{thm:PCstarcorrect}
The output of \Cref{alg:PCstar} is a CPDAG with the same undirected skeleton, unshielded colliders, and orientable induced cycles as $\cgtr$. Without the optional cycle orientation step, its complexity is $\mathcal{O}(n^{d+2})$, where $n = |V|$ and $d = \max_{v \in V} \indeg(v)$. 
    
\end{theorem}

\begin{proof} 
The correctness of the output is a direct consequence of \Cref{thm:PCstarcorrect}, \Cref{prop:findcoll} and \Cref{lem:cycleori}. 
For the complexity, recall that the worst-case runtime of the traditional PC algorithm is that of \Cref{alg:PCalgskel}, which lies in $\mathcal{O}(n^{d+2})$, see \cite[p. 110]{spirtes_causation_1993}. To prove that \texttt{PCstar} without cycle orientation has the same complexity, it suffices to show that none of the computations in lines 1,2 and 5 of \Cref{alg:PCstar} have poorer worst-case performance. Line 1 is exactly \Cref{alg:PCalgskel}. The complexity of Line 2 is the same as that of the collider orientation step of the original PC algorithm by \Cref{rmk:collidercomplexity}, and Line 5 is the same edge orientation step as that of PC. Both of these procedures are polynomial in $|E|$ and in particular bounded above by $n^{d+2}$ \cite[p. 143]{meek_graphical_2023}.
\end{proof}

The worst-case performance of detecting all orientable induced cycles is $\mathcal{O}(d^n)$; this bound comes from finding all cycles which contain a given unshielded triple. This computation becomes intractable even for relatively small $n$. However, the actual orientation of an orientable induced cycle with $|W|$ nodes involves querying the set $\globalCsep{\cg}{C}$ a total of $|W|^2$ times. This suggests an alternative approach which replaces Lines 3-4 of \Cref{alg:PCstar} with the orientation of only certain cycles which are selected manually from the output of lines 1-2. For example, to determine the orientation of a single undirected edge of the CPDAG, one could determine only the induced cycles which contain it and check for orientability by verifying the conditions in \Cref{misc:orientabilityconditions}, but we leave this to future work.

\section{Implementation and Outlook}\label{sec:simulations}
We implemented \texttt{PCstar} in julia \cite{bezanson2017julia}, using the packages Graphs.jl \cite{Graphs2021} and OSCAR.jl \cite{OSCAR} for graph and tropical arithmetic functionality respectively. Methods for applying the edge orientation rules were imported from the CausalInference.jl \cite{CausalInference} package. The source code may be found at \url{https://github.com/fpnowell/starskeleton}.

We ran our function \texttt{PCstar} on data produced by randomly generated weighted DAGs on 10, 15 and 20 nodes of varying degrees of sparsity. For the 15 node graphs onwards, we did not generate the entire set of CI statements in the corresponding model, opting instead to query the model as needed by checking for $C^\ast$- separation in the true DAG. 
The runtime of our implementation increases rapidly with increasing $n$ and maximal in-degree $d$. Despite this, our algorithm terminated in under 24 hours on 31-node graphs of in-degrees 2 and 3; quantities which are motivated by the real-world data set of the Danube hidden river network commonly used in extremal statistics \cite{asadi2015extremes}. 

The main obstacle to the real-world application of \texttt{PCstar} is the lack of a canonical conditional independence test. While linear SEMs are inherently parametric, meaning that a CI test can be obtained by imposing rank conditions on an estimated covariance matrix, no similar procedure exists for MLBNs. Thus, future work may involve the development of specialized non-parametric conditional independence testing \cite{shah_CItesting} for MLBNs. Other potential avenues include a combination of our methods with recent work on edge weight estimation of MLBNs \cite{adams_ferry2025} and the development of alternative causal discovery tools for MLBNs based on  methods such as greedy equivalence search \cite{chickering2003optimal} and \enquote{do}-calculus \cite{Pearl_2009}.
\begin{table}[H]
    \caption{Average values for the output of \texttt{PCstar} on randomly generated weighted DAGs of varying size and sparsity level. The quantity $d$ is the maximal in-degree of the true DAG. }
    \label{tab:PCstarresults}
    \centering
    \begin{tabular}{cccccc}
        \toprule
        $|V|$ &  $ \ \ d \ \ $  & \# edges  & \# edges & \# recovered edges  & \# recovered edges \\
        \phantom{} & \phantom{} & of $\cg$  & of $\cgtr$  &(w/o cycle orientation) & (with cycle orientation) \\
        \midrule
        10 & 2 & 9.17 & 8.66 & 3.88 & 4.43 \\

        \phantom{} & 3 & 13.13 & 10.54 & 6.38 & 7.54 \\
        \phantom{} & 4 & 16.13 & 11.57 & 6.99 & 8.33 \\
        \phantom{} & 5 & 19.63 & 12.15 & 6.94 & 8.60 \\
        \midrule

        15 & 2 & 15.63 & 14.89 & 8.53 & 9.58 \\
        \phantom{} & 3 & 17.46 & 15.95 & 10.02 & 11.31 \\
        \phantom{} & 4 & 22.74 & 18.21 & 12.83 & 14.59 \\
        \phantom{} & 5 & 28.30 & 19.61 & 14.17 & 15.92\\
        \midrule
        20 & 2 &  20.36 & 19.74  & 11.51 & 12.47 \\
        \phantom{} & 3 & 22.57 & 21.31 & 13.74 & 15.32 \\
        \phantom{} & 4 & 28.13 & 24.38 & 17.85 & 20.14 \\

        \bottomrule
    \end{tabular}
\end{table}

\section{Acknowledgments}
Benjamin Hollering was partially supported by the Alexander von Humboldt foundation. 
Francesco Nowell is supported by the Deutsche Forschungsgemeinschaft (DFG, German Research Foundation)
under the Priority Programme “Combinatorial Synergies” (SPP 2458, project ID: 539875257).

\printbibliography
\end{document}